\documentclass{article}
\usepackage{ijcai22}

\usepackage{times}

\usepackage{soul}
\usepackage{url}
\usepackage[hidelinks]{hyperref}
\usepackage[utf8]{inputenc}
\usepackage[small]{caption}
\usepackage{graphicx}
\usepackage{amsmath}
\usepackage{booktabs}
\usepackage{color}
\usepackage{amsfonts}
\usepackage{amsmath}
\usepackage{amsthm}
\usepackage{amssymb}
\usepackage{thmtools,thm-restate}

\usepackage{algorithm2e}
\let\oldnl\nl
\newcommand{\nonl}{\renewcommand{\nl}{\let\nl\oldnl}}%

\newenvironment{tightlist}%
{\begin{list}{$\bullet$}{%
    \setlength{\topsep}{0in}
    \setlength{\partopsep}{0in}
    \setlength{\itemsep}{0in}
    \setlength{\parsep}{0in}
    \setlength{\leftmargin}{1.5em}
    \setlength{\rightmargin}{0in}
}
}%
{\end{list}
}

\newcommand{\Ours}{PG3}

\renewcommand{\O}{\mathcal{O}}

\newtheorem{defn}{Definition}

\newtheorem{assumption}{Assumption}
\counterwithin*{thm}{section}

\title{\Ours{}: Policy-Guided Planning for Generalized Policy Generation}

\author{
Ryan Yang\footnote{Equal contribution.}\and
Tom Silver$^*$\and
Aidan Curtis\and
Tom\'{a}s Lozano-P\'{e}rez\And
Leslie Pack Kaelbling
\\
\affiliations
MIT Computer Science and Artificial Intelligence Laboratory\\
\emails
\{ryanyang, tslvr, curtisa, tlp, lpk\}@mit.edu
}

\begin{document}

\maketitle

\begin{abstract}
A longstanding objective in classical planning is to synthesize policies that generalize across multiple problems from the same domain. In this work, we study generalized policy search-based methods with a focus on the score function used to guide the search over policies. We demonstrate limitations of two score functions and propose a new approach that overcomes these limitations. The main idea behind our approach, Policy-Guided Planning for Generalized Policy Generation (PG3), is that a candidate policy should be used to guide planning on training problems as a mechanism for evaluating that candidate. Theoretical results in a simplified setting give conditions under which PG3 is optimal or admissible. We then study a specific instantiation of policy search where planning problems are PDDL-based and policies are lifted decision lists. Empirical results in six domains confirm that PG3 learns generalized policies more efficiently and effectively than several baselines. Code: \href{https://github.com/ryangpeixu/pg3}{https://github.com/ryangpeixu/pg3}
\end{abstract}

\section{Introduction}
\label{sec:intro}

How can we compile a transition model and a set of training tasks into a reactive policy?
Can these policies generalize to large tasks that are intractable for modern planners?
These questions are of fundamental interest in AI planning \cite{triangletables}, with progress in \emph{generalized planning} recently accelerating~\cite{srivastava2011foundations,bonet2015policies,jimenez2019review,rivlin2020generalized}.

\emph{Generalized policy search} (GPS) is a flexible paradigm for generalized planning~\cite{levine2003learning,segovia2021generalized}.
In this family of methods, a search is performed through a class of generalized (goal-conditioned) policies, with the search informed by a \emph{score function} that maps candidate policies to scalar values.
While much attention has been given to different policy representations, there has been relatively less work on the score function.
The score function plays a critical role: if the scores are uninformative or misleading, the search will languish in less promising regions of policy space.

In this work, we propose Policy-Guided Planning for Generalized Policy Generation (\Ours{}), which centers around a new score function for GPS.
Given a candidate policy, \Ours{} solves the set of training tasks \emph{using the candidate to guide planning}.
When a plan is found, the agreement between the plan and the policy contributes to the score.
Intuitively, if the policy is poor, the planner will ignore its guidance or fail to find a plan, and the score will be poor; if instead the policy is nearly able to solve problems, except for a few ``gaps'', the planner will rely heavily on its guidance, for a good score.

\begin{figure*}[t]
  \centering
    \noindent
    \includegraphics[width=\textwidth]{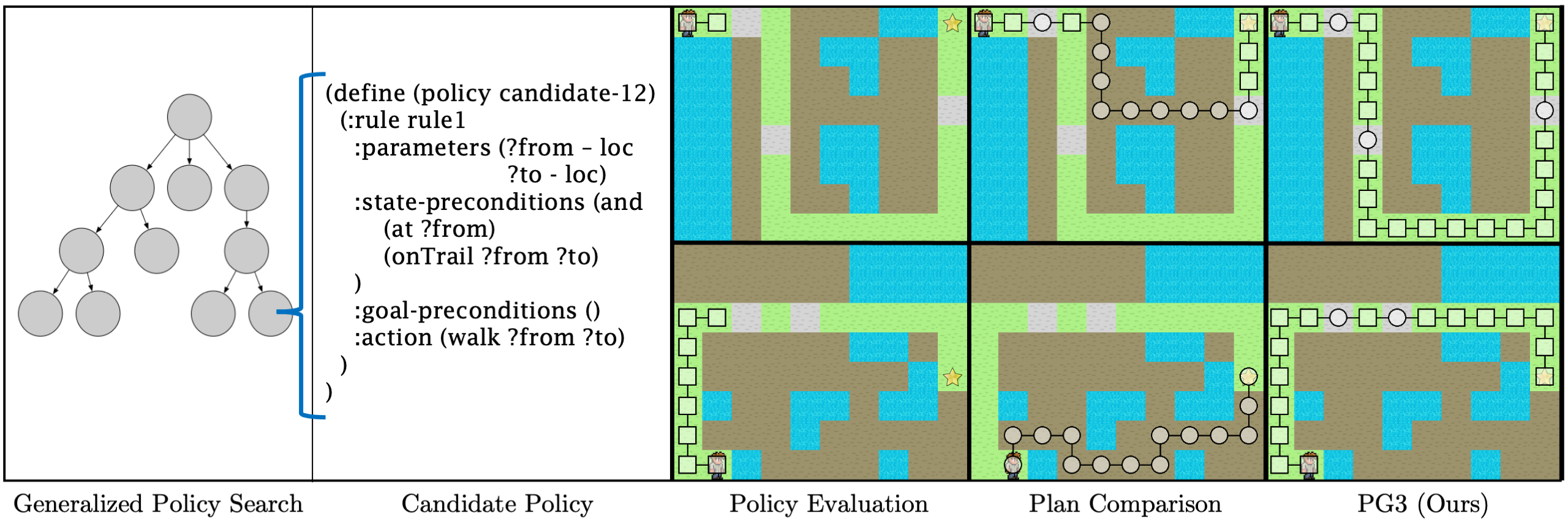}
    \caption{(Left) In this work, we study generalized policy search (GPS). Each node represents a candidate policy. (Left middle) A candidate policy for the Forest domain. In this domain, the agent \emph{can} walk along the marked trail (green), but it can also walk in the dirt (brown) as long as it avoids water (blue). The candidate policy, which has the agent walk along the trail, would always succeed, except that the trails sometimes contain rocks (gray), which must be climbed; the policy thus fails when a rock is encountered.
    This policy is ``close'' to satisficing (see Definition~\ref{defn:costtogo} for a precise meaning), so we would like to give it a good score.
    Each of the next three columns illustrates a different score function evaluated using two training problems (top and bottom rows).
    Squares (circles) denote agreement (disagreement) between the candidate policy and plan.
    (Middle) Policy evaluation gives a trivial score to the candidate policy because it never reaches the goal. (Middle right) Plan comparison gives a poor score because the policy differs substantially from the paths found by a planner, which often leave the marked trail. (Right) We propose \Ours{}, which evaluates the candidate by policy-guided planning, resulting in a good score.}
  \label{fig:teaser}
\end{figure*}

\Ours{} combines ideas from two score functions for GPS: \emph{plan comparison}, which plans on the training tasks and records the agreement between the found plans and the candidate policy; and \emph{policy evaluation}, which executes the policy on the training tasks and records the number of successes.
While plan comparison provides a dense scoring over policy space, it works best when the solutions for a set of problems all conform to a simple policy.
Absent this property, GPS can overfit to complicated and non-general solutions.
Policy evaluation scores are extremely sparse within the policy space, effectively forcing an exhaustive search until reaching a region of the policy space with non-zero scores. 
However, the policies resulting from this approach are significantly more compact and general.
\Ours{} combines the strengths of plan comparison and policy evaluation to overcome their individual limitations.
See Figure~\ref{fig:teaser} for an illustration.

In experiments, we study a simple but expressive class of \emph{lifted decision list} generalized policies for PDDL domains  \cite{mooney1995induction,khardon1999learning,levine2003learning}.
Since the policies are lifted, that is, invariant over object identities, they can generalize to problems involving new and more objects than seen during learning.
We propose several \emph{search operators} for lifted decision list policies, including a ``bottom-up'' operator that proposes policy changes based on prior planning experience.

This paper makes the following main contributions: (1) we propose \Ours{} as a new approach to GPS; (2) we provide conditions under which \Ours{} is optimal or admissible for policy search; (3) we propose a specific instantiation of GPS for PDDL domains; and (4) we report empirical results across six PDDL domains, demonstrating that \Ours{} efficiently learns policies that generalize to held-out test problems.
\section{Related Work}
\label{sec:related}

Policy search has been previously considered as an approach to generalized planning~\cite{levine2003learning,khardon1999learning,jimenez2015computing,segovia2018computing}.
Most prior work either uses \emph{policy evaluation} as a score function, or relies on classical planning heuristics for the search through policy space.
One exception is Segovia-Aguas et al.~\shortcite{segovia2021generalized}, who extend policy evaluation with goal counting; we include this baseline in our experiments.
Because of the limited search guidance and size of the policy space, all of these methods can typically learn only small policies, or require additional specifications of the problems that the policy is expected to solve ~\cite{Illanes_McIlraith_2019}. 

Another strategy for generalized planning is to construct a policy, often represented as a finite state machine, from example plans~\cite{kplanner,fsa_planner,Srivastava2011DirectedSF,winner2008dsplanner}. 
Such approaches are successful if the plans all conform to a single compact policy. 
This assumption is violated in many interesting problems including ones that we investigate here (c.f. \emph{plan comparison}).
Other approaches have been proposed to learn \emph{abstractions} from example plans that lend themselves to compact policies ~\cite{genplan_representation,learning_abstractions_3}. While these abstraction-based approaches expand the space of problems that can be solved, they are still unable to learn policies that deviate dramatically from the example plans.

Recent work has proposed using deep learning to learn a generalized policy \cite{groshev2017learning} or a heuristic function \cite{rivlin2020generalized,karia2020learning}.
Inspired by these efforts, we include a graph neural network baseline in experiments.
Beyond classical planning in factored and relational domains, learning a goal-conditioned policy is of interest in reinforcement learning and continuous control~\cite{gcrl2,gcrl3}.
Unlike in generalized planning, these settings typically do not assume a known, structured domain model, and often make very different assumptions about state and action representations.
The extent to which our insights in this work can be applied in these settings is an interesting area for future investigation.

We propose policy-guided planning as a mechanism for policy \emph{learning}.
Previous work has also considered the question of how to use a \emph{fixed} (often previously learned) policy to aid planning \cite{khardon1999learning,yoon2008learning}.

\section{Preliminaries}
We begin with a brief review of classical planning and then define the generalized problem setting.

\subsection{Classical Planning}

In classical planning, we are given a \emph{domain} and a \emph{problem}, both often expressed in PDDL \cite{pddl}.
We use the STRIPS subset with types and negative preconditions.

A domain is a tuple $\langle P, A \rangle$ where $P$ is a set of \emph{predicates} and $A$ is a set of \emph{actions}.
A predicate $p \in P$ consists of a name and an arity. An \emph{atom} is a predicate and a tuple of \emph{terms}, which are either \emph{objects} or \emph{variables}.
For example, $\texttt{on}$ is a predicate with arity 2; $\texttt{on(X, Y)}$ and $\texttt{on(b1, b2)}$ are atoms, where $\texttt{X}, \texttt{Y}$ are variables and $\texttt{b1}, \texttt{b2}$ are objects.
Terms may be typed.
A \emph{literal} is an atom or its negation.
An action $a \in A$ consists of a name and a tuple $\langle \textsc{Par}(a), \textsc{Pre}(a), \textsc{Add}(a), \textsc{Del}(a) \rangle$, which are the parameters, preconditions, add effects, and delete effects of the action respectively.
The parameters are a tuple of variables.
The preconditions are a set of literals, and the add and delete effects are sets of atoms over the parameters.
A \emph{ground action} is an action with objects substituted for parameters.
For example, if $\texttt{move}$ is an action with parameters $\texttt{(?from, ?to)}$, then $\texttt{move(l5, l6)}$ is a ground action.
In general, actions may be associated with variable costs; in this work, we focus on \emph{satisficing} planning and assume all costs are unitary.

A domain is associated with a set of problems.
A problem is a tuple $\langle O, I, G \rangle$ where $O$ is a set of objects, $I$ is an \emph{initial state}, and $G$ is a \emph{goal}.
States and goals are sets of atoms with predicates from $P$ instantiated with objects from $O$.
Given a state $S$, a ground action $\underline{a} = a(o_1, \dots, o_k)$ with $a \in A$ and $o_i \in O$ is \emph{applicable} if $\textsc{Pre}(\underline{a}) \subseteq S$.\footnote{That is, if all positive atoms in $\textsc{Pre}(\underline{a})$ are in $S$, and no negated atoms in $\textsc{Pre}(\underline{a})$ are in $S$. We use this shorthand throughout.}
Executing an applicable action $\underline{a}$ in a state $S$ results in a \emph{successor state}, denoted $\underline{a}(S) = (S \setminus \textsc{Del}(\underline{a})) \cup \textsc{Add}(\underline{a})$.
Given a problem and domain, the objective is to find a \emph{plan} $(\underline{a}_1, \dots, \underline{a}_m)$ that \emph{solves} the problem, that is, $\underline{a}_m \circ \cdots \circ \underline{a}_1(I) \subseteq G$ and each $\underline{a}_i$ is applicable when executed.

\subsection{Problem Setting: Generalized Planning}
\label{sec:setting}

Classical planning is traditionally concerned with solving individual planning problems.
In \emph{generalized planning}, the objective instead is to find a unified solution to multiple problems from the same domain.
Here we are specifically interested in learning a \emph{generalized policy} $\pi$, which maps a state $S$ and a goal $G$ to a ground action $\underline{a}$, denoted $\pi(S, G) = \underline{a}$.
This form of $\pi$ is very general, but our intention is to learn a \emph{reactive} policy, which produces an action with minimal computation, and does not, for example, plan internally.
Given a problem $\langle O, I, G \rangle$, the policy is said to \emph{solve} the problem if there exists a plan $(\underline{a}_1, \dots, \underline{a}_m)$ that solves the problem, and such that $\pi(S_i, G) = \underline{a}_i$ for each state $S_i$ in the sequence of successors.
In practice, we evaluate a policy for a maximum number of steps to determine if it solves a problem.

Our aim is to learn a policy that generalizes to many problems from the same domain, including problems that were not available during learning.
We therefore consider a problem setting with a \emph{training} phase and a \emph{test} phase.
During training, we have access to the domain $\langle P, A \rangle$ and a set of \emph{training problems} $\Psi = \{\langle O_1, I_1, G_1 \rangle, \dots, \langle O_n, I_n, G_n \rangle \}$.
The output of training is a single policy $\pi$.
During the test phase, $\pi$ is evaluated on a set of held-out \emph{test problems}, often containing many more objects than those seen during training.
The objective of generalized policy learning is to produce a policy $\pi$ that solves as many test problems as possible.
A policy that solves all problems in the test set is referred to as \emph{satisficing}.

\section{Policy-Guided Planning for Policy Search}

In this work, we build on \emph{generalized policy search} (GPS), where a search through a set of generalized policies is guided by a \emph{score function}.
Specifically, we perform a greedy best-first search (GBFS), exploring policies in the order determined by  $\textsc{Score}$, which takes in a candidate policy $\pi$, the domain $\langle P, A \rangle$, and the training problems $\Psi$, and returns a scalar value, where lower is better (Algorithm~\ref{alg:pg3}).

\emph{Example 1 (Forest).} Forest is a PDDL domain illustrated in Figure~\ref{fig:teaser}; see the caption for a description.
There is a simple satisficing policy for this domain: walk along the trail, and climb over any rocks that are encountered.
This policy is similar in spirit to the classic wall-following policy for mazes.
The good-but-imperfect candidate policy illustrated in Figure~\ref{fig:teaser} walks along the trail, but does not climb.
When this candidate is considered during GPS, it will spawn many successor policies: one will add \texttt{(isRock ?to)} as a precondition to \texttt{rule1}; another will create a rule with the \texttt{climb} action; among others.
Each successor will then be scored.

The score function has a profound impact on the efficacy of GPS.
One possible score function is \textbf{policy evaluation}: the candidate policy is executed in each training problem, and the score is inversely proportional to the number of problems solved.
A major limitation of this score function is that its outputs are trivial for all policies that do not completely solve any training problems, such as the policy described above.

Another possible score function is \textbf{plan comparison}: a planner is used to generate plans for the training problems, and the candidate policy is scored according to the agreement between the plans and the candidate policy (Algorithm~\ref{alg:plancompare}).
When there are multiple ways to solve a problem, this score function can sharply mislead GPS.
For example, plan comparison gives a poor score to the follow-and-climb policy in Example 1, even though the policy is satisficing!
This issue is not limited to the Forest domain; similar issues arise whenever goal atoms can be achieved in different orders, or when the same state can be reached through two different paths from the initial state.
This phenomenon also arises in ``bottom-up''  generalized planning approaches (Section~\ref{sec:related}).




Our main contribution is \textbf{Policy-Guided Planning for Generalized Policy Generation (\Ours{})}.
Given a candidate policy, \Ours{} runs \emph{policy-guided planning} on each training problem (Algorithm~\ref{alg:policyguidedplanning}).
Our implementation of policy-guided planning is a small modification to $\text{A}^*$ search: for each search node that is popped from the queue, in addition to expanding the standard single-step successors, we roll out the policy for several time steps (maximum 50 in experiments\footnote{We did not exhaustively sweep this hyperparameter, but found in preliminary experiments that 50 was better than 10.}, or until the policy is not applicable), creating search nodes for each state encountered in the process.
The nodes created by policy execution are given cost 0, in contrast to the costs of the single-step successors, which come from the domain.
These costs encourage the search to prefer paths generated by the policy.

For each training problem, if policy-guided planning returns a plan, we run \emph{single plan comparison} (Algorithm~\ref{alg:plancompare}) to get a score for the training problem.
For each state and action in the plan, the policy is evaluated at that state and compared to the action; if they are not equal, the score is increased by 1.
If a plan was not found, the maximum possible plan length $\ell$ is added to the overall score.
Finally, the per-plan scores are accumulated to arrive at an overall policy score.
To accumulate, we use \emph{max} for its theoretical guarantees (Section \ref{sec:theory}), but found \emph{mean} to achieve similar empirical performance.

Intuitively, if a candidate policy is able to solve a problem except for a few ``gaps'', policy-guided planning will rely heavily on the policy's suggestions, and the policy will obtain a high score from \Ours{}.
If the candidate policy is poor, policy-guided planning will ignore its suggestions, resulting in a low score.
For example, consider again the good-but-imperfect policy for the Forest domain from Example 1.
Policy-guided planning will take the suggestions of this policy to follow the marked trail until a rock is encountered, at which point the policy becomes stuck.
The planner will then rely on the single-step successors to climb over the rock, at which point it will again follow the policy, continuing along the trail until another rock is encountered or the goal is reached.


\begin{algorithm}[t]
  \SetAlgoNoEnd
  \DontPrintSemicolon
  \SetKwFunction{algo}{algo}\SetKwFunction{proc}{proc}
  \SetKwProg{myalg}{}{}{}
  \SetKwProg{myproc}{Subroutine}{}{}
  \SetKw{Continue}{continue}
  \SetKw{Break}{break}
  \SetKw{Return}{return}
  \myalg{\textsc{Generalized Policy Search via GBFS}}{
    \nonl \textbf{input:} domain $\langle P, A \rangle$ and training problems $\Psi$\;
    \nonl \textbf{input:} search operators $\Omega$\;
    \tcp{\footnotesize E.g., an empty decision list}
    \nonl \textbf{initialize:} trivial generalized policy $\pi_0$\;
    \tcp{\footnotesize Ordered low-to-high priority}
    \nonl \textbf{initialize:} empty priority queue $q$\;
    \tcp{\footnotesize See Algorithm~\ref{alg:pg3score}}
    \nonl Push $\pi_0$ onto $q$ with priority $\textsc{Score}(\pi_0, P, A, \Psi)$\;
    \tcp{\footnotesize Repeat until max iters}
    \For{$i = 1, 2, 3, ...$} {
    \nonl Pop $\pi$ from $q$
    \tcp*{\footnotesize Best policy in $q$}
    \nonl \For{search operator $\omega \in \Omega$}{
    \nonl \For{$\pi' \in \omega(\pi, P, A, \Psi)$} {
    \nonl Push $\pi'$ to $q$ with $\textsc{Score}(\pi', P, A, \Psi)$\;
    }
    }
    }
    \tcp{\footnotesize Policy with the lowest score}
    \Return Best seen policy $\pi^*$\;
    }
\caption{\small{Generalized policy search via GBFS. See Algorithm~\ref{alg:pg3score} for the \Ours{}-specific implementation of \textsc{Score}.}}
\label{alg:pg3}
\end{algorithm}

\begin{algorithm}[t]
  \SetAlgoNoEnd
  \DontPrintSemicolon
  \SetKwFunction{algo}{algo}\SetKwFunction{proc}{proc}
  \SetKwProg{myalg}{}{}{}
  \SetKwProg{myproc}{Subroutine}{}{}
  \SetKw{Continue}{continue}
  \SetKw{Break}{break}
  \SetKw{Return}{return}
  \myalg{\textsc{\Ours{} Score Function for Policy Search}}{
    \nonl \textbf{input:} candidate policy $\pi$\;
    \nonl \textbf{input:} domain $\langle P, A \rangle$ and training problems $\Psi$\;
    \nonl \textbf{hyperparameter:} max plan horizon $\ell$\;
    \nonl \textbf{initialize:} $scores \gets []$\;
    \nonl \For{$\langle O, I, G \rangle \in \Psi$} {
    \tcp{\footnotesize See Algorithm \ref{alg:policyguidedplanning}}
    \nonl Run $\textsc{PolicyGuidedPlan}(\pi, O, I, G, P, A)$\;
    \If {a plan $p$ is found}{
    \tcp{\footnotesize See Algorithm~\ref{alg:plancompare}} $score \gets \textsc{PlanComparison}(\pi, p, I)$\;
    $scores.append(score)$\;}
    \Else {
    $scores.append(\ell)$}
    }
    \Return $max(scores)$\;
    }
\caption{\small{Scoring via policy-guided planning} (\Ours{}).}
\label{alg:pg3score}
\end{algorithm}

\begin{algorithm}[t]
  \SetAlgoNoEnd
  \DontPrintSemicolon
  \SetKwFunction{algo}{algo}\SetKwFunction{proc}{proc}
  \SetKwProg{myalg}{}{}{}
  \SetKwProg{myproc}{Subroutine}{}{}
  \SetKw{Continue}{continue}
  \SetKw{Break}{break}
  \SetKw{Return}{return}
  \myalg{\textsc{Policy-Guided Planning}}{
    \nonl \textbf{input:} policy $\pi$\;
    \nonl \textbf{input:} problem $\langle O, I, G \rangle$ and domain $\langle P, A \rangle$\;
    \nonl \textbf{hyperparameter:} max plan horizon $\ell$\;
    \tcp{\footnotesize See definition below}
    \nonl \Return $\text{A}^*$ with \textsc{GetSuccessors}\;
    }
  \myproc{\textsc{GetSuccessors}}{
    \nonl \textbf{input:} state $S$
    \tcp*{\footnotesize State in search}
    \nonl \textbf{hyperparameter:} max policy execution steps $k$\;
    \nonl \For{ground actions $\underline{a}$ applicable in $S$} {
    \tcp{\footnotesize Standard successor function}
    \nonl \textbf{yield} $\underline{a}(S)$ with cost from domain
    }
    \tcp{\footnotesize Policy-guided successors}
    \nonl \For{$i=1, 2, \dots, k$} {
    \nonl if $\pi$ is not applicable in $S$, \textbf{stop}.\;
    \tcp{\footnotesize Get action from policy}
    \nonl $\underline{a}_i \gets \pi(S, G)$\;
    \tcp{\footnotesize Update current state}
    \nonl $S \gets \underline{a}_i(S)$\;
    \tcp{\footnotesize Yield all encountered states}
    \nonl \textbf{yield} $S$ with cost 0\;
    }
  }
\caption{\small{Policy-guided planning.} A helper for the \Ours{} score function (Algorithm~\ref{alg:pg3score}). In our experiments with PDDL domains, $\text{A}^*$ uses the $\text{h}_{\text{Add}}$ heuristic~\protect\cite{bonet2001planning}.}
\label{alg:policyguidedplanning}
\end{algorithm}

\begin{algorithm}[t]
  \SetAlgoNoEnd
  \DontPrintSemicolon
  \SetKwFunction{algo}{algo}\SetKwFunction{proc}{proc}
  \SetKwProg{myalg}{}{}{}
  \SetKwProg{myproc}{Subroutine}{}{}
  \SetKw{Continue}{continue}
  \SetKw{Break}{break}
  \SetKw{Return}{return}
  \myalg{\textsc{Single Plan Comparison}}{
    \nonl \textbf{input:} policy $\pi$, plan $p$, initial state $I$\;
    \nonl \textbf{initialize:} $score$ to 0 and $S$ to $I$\;
    \For{$\underline{a}$ in $p$}{
    \If{$\pi(S) \neq \underline{a}$}{
        \nonl Add $1$ to $score$\;
    }
    \nonl $S \gets \underline{a}(S)$
    }
    \Return $score$\; 
    }
\caption{\small{Scoring a policy by comparison to a single plan.} A helper for the \Ours{} score function (Algorithm~\ref{alg:pg3score}).}
\label{alg:plancompare}
\end{algorithm}

\textbf{Limitation.}
\Ours{} requires planning during scoring, which can be computationally expensive.
Nonetheless, in experiments, we will see that \Ours{} (implemented in Python) can quickly learn policies that generalize to large test problems.

\subsection{Theoretical Results}
\label{sec:theory}
We now turn to a simplified setting for theoretical analysis.

\emph{Example 2 (Tabular).}
Suppose we were to represent policies as tables.
Let $\{(S, G) \mapsto \underline{a}\}$ denote a policy that assigns action $\underline{a}$ for state $S$ and goal $G$.
Consider a single search operator for GPS, which changes one entry of a policy's table in every possible way.
For example, in a domain with two states $S_1, S_2$, one goal $G$, and two actions $\underline{a}_1, \underline{a}_2$, the GPS successors of the policy $\{(S_1, G) \mapsto \underline{a_1}, (S_2, G) \mapsto \underline{a_1}\}$ would be $\{(S_1, G) \mapsto \underline{a_2}, (S_2, G) \mapsto \underline{a_1}\}$ and $\{(S_1, G) \mapsto \underline{a_1}, (S_2, G) \mapsto \underline{a_2}\}$.
Here we are not concerned with generalization; training and test problems are the same, $\Psi$.

Our main theoretical results concern the influence of the \Ours{} score function on GPS.
Proofs are included in Appendix~\ref{sec:proofs}.
We begin with a definition.

\begin{defn}[GPS cost-to-go]
\label{defn:costtogo}
The \emph{GPS cost-to-go} from a policy $\pi_0$ is the minimum $k$ s.t. there exists sequences of policies $\pi_0, \dots, \pi_{k}$ and search operators $\omega_0, \dots, \omega_{k-1}$ s.t. $\pi_{k}$ is satisficing, and $\forall j, \pi_{j+1} \in \omega_j(\pi_j)$.
\end{defn}

Note that in the tabular setting, the GPS cost-to-go is equal to the minimum number of entries in the policy table that need to be changed to create a policy that solves all problems.



\begin{assumption}
\label{assumption:allplanscomparison}
The heuristic used for $\text{A}^*$ search in policy-guided planning (Algorithm~\ref{alg:policyguidedplanning}) is admissible; planning is complete; and all costs in the original problem are unitary.\footnote{Unitary costs are assumed for simplicity; the proofs can be extended to handle any positive (nonzero) costs.}
\end{assumption}

\begin{restatable}{thm}{equalthm}
\label{thm:equal}
Under Assumption~\ref{assumption:allplanscomparison}, in the tabular setting (Example 2), if a policy $\pi$ solves all but one of the problems in $\Psi$, then the \Ours{} score for $\pi$ is equal to the GPS cost-to-go.
\end{restatable}

As a corollary, GPS with \Ours{} will perform optimally under the conditions of Theorem~\ref{thm:equal}, in terms of the number of nodes expanded by GBFS before reaching a satisficing policy.

\begin{restatable}{thm}{admissiblethm}
\label{thm:costtogo}
Under Assumption~\ref{assumption:allplanscomparison}, in the tabular setting (Example 2), \Ours{} is a lower bound on the GPS cost-to-go.
\end{restatable}

These results do not hold for other choices of score functions, including policy evaluation or plan comparison.
However, the results do not immediately extend beyond the tabular setting; in general, a single application of a search operator could change a policy's output on every state, converting a ``completely incorrect'' policy into a satisficing one.
In practice, we expect GPS to lie between these two extremes, with search operators often changing a policy on more than one, but far fewer than all, states in a domain.
Toward further understanding the practical case, we next consider a specific instantiation of GPS that will allow us to study \Ours{} at scale.

\subsection{Generalized Policies as Lifted Decision Lists}

We now describe a hypothesis class of \emph{lifted decision list} generalized policies that are well-suited for PDDL domains \cite{mooney1995induction,levine2003learning}.
\begin{defn}[Rule]
\label{def:rule}
A \emph{rule} $\rho$ for a domain $\langle P, A \rangle$ is a tuple:
\begin{itemize}
    \item $\textsc{Par}(\rho)$: \emph{parameters}, a tuple of variables;
    \item $\textsc{Pre}(\rho)$: \emph{preconditions}, a set of literals;
    \item $\textsc{Goal}(\rho)$: \emph{goal preconditions}, a set of literals;
    \item $\textsc{Act}(\rho)$: the rule's \emph{action}, from $A$;
\end{itemize}
with all literals and actions instantiated over the parameters, and with all predicates in $P$.
\end{defn}
As with actions, rules can be \emph{grounded} by substituting objects for parameters, denoted $\underline{\rho} = \rho(o_1, \dots, o_k)$.
Intuitively, a rule represents an existentially quantified conditional statement: if there exists some substitution for which the rule's preconditions hold, then the action should be executed.
We formalize these notions with the following two definitions.
\begin{defn}[Rule applicability]
\label{def:ruleapplicable}
Given a state $S$ and goal $G$ over objects $\O$, a rule $\rho$ is \emph{applicable} if $\exists (o_1, \dots, o_k)$ s.t. $\textsc{Pre}(\underline{\rho}) \subseteq S$ and $\textsc{Goal}(\underline{\rho}) \subseteq G$, where $\underline{\rho} = \rho(o_1, \dots, o_k)$ and with each $o_i \in \O$.
\end{defn}
\begin{defn}[Rule execution]
\label{def:ruleexecution}
Given a state $S$ and goal $G$ where rule $\rho$ is applicable, let $(o_1, \dots, o_k)$ be the first\footnote{First under a fixed, arbitrary ordering of object tuples, e.g., lexicographic, to avoid nondeterminism.} tuple of objects s.t. the conditions of Definition~\ref{def:ruleapplicable} hold with $\underline{\rho} = \rho(o_1, \dots, o_k)$. Then the \emph{execution} of $\rho$ in $(S, G)$ is $\underline{a} = \textsc{Act}(\underline{\rho})$, denoted $\rho(S, G) = \underline{a}$.
\end{defn}
Rules are the building blocks for our main generalized policy representation, the lifted decision list.
\begin{defn}[Lifted decision list policy]
\label{def:ldl}
A \emph{lifted decision list policy} $\pi$ is an (ordered) list of rules $[\rho_1, \rho_2, \dots, \rho_\ell]$.
Given a state $S$ and goal $G$, $\pi$ is \emph{applicable} if $\exists i$ s.t. $\rho_i$ is applicable.
If $\pi$ is applicable, the \emph{execution} of $\pi$ is the execution of the first applicable rule, denoted $\pi(S, G) = \underline{a}$.
\end{defn}

See Appendix~\ref{sec:learnedpolicies} for examples of lifted decision list policies represented using PDDL-like syntax.
Compared to PDDL operators, in addition to the lack of effects and the addition of goal preconditions, it is important to emphasize that the rules are \emph{ordered}, with later rules only used when previous ones are not applicable.
Also note these policies are \emph{partial}: they are only defined when they are applicable.

\textbf{Representational Capacity.} We selected lifted decision lists because they are able to compactly express a rich set of policies across several domains of interest.
For example, repeated application of a lifted decision list policy can lead to looping behavior.
Nonetheless, note the absence of numeric values, transitive closures, universal quantifiers, and more sophisticated control flow and memory~\cite{jimenez2015computing,bonet2019learning}.
\Ours{} is not specific to lifted decision lists and could be used with richer policy classes.

\begin{table*}
	\centering
	\footnotesize
	\begin{tabular}{| l | p{0.7cm} | p{0.7cm} | p{0.7cm} | p{0.7cm} | p{0.7cm} | p{0.7cm} | p{0.7cm} | p{0.7cm} | p{0.7cm} | p{0.7cm} | p{0.7cm} | p{0.7cm} | p{0.7cm} | p{0.7cm} | }
	\hline
	\multicolumn{1}{|c|}{} &\multicolumn{2}{c|}{PG3 (Ours)} &
	\multicolumn{2}{c|}{Policy Eval} &
	\multicolumn{2}{c|}{Plan Comp} &
	\multicolumn{2}{c|}{Combo} &
	\multicolumn{2}{c|}{Goal Count} &
	\multicolumn{2}{c|}{GNN BC} &
	\multicolumn{2}{c|}{Random} \\
	\hline
	Domains &
	\;{\scriptsize Eval} & \:{\scriptsize Time} &
	\;{\scriptsize Eval} & \:{\scriptsize Time} &
	\;{\scriptsize Eval} & \:{\scriptsize Time} &
	\;{\scriptsize Eval} & \:{\scriptsize Time} &
	\;{\scriptsize Eval} & \:{\scriptsize Time} &
	\;{\scriptsize Eval} & \:{\scriptsize Time} &
	\;{\scriptsize Eval} & \:{\scriptsize Time} \\
	\hline
	Delivery & \bf{1.00} & 1.5 & 0.00 & \;\;\;-- & 0.10 & 5.8 & 0.20 & 2132.4 & \bf{1.00} & 27.9 & 0.40 & 39.9 & 0.00 & \;\;\;-- \\
	Forest & \bf{1.00} & 107.2 & \bf{1.00} & 605.1 & 0.16 & 51.5 & \bf{1.00} & 815.4 & 0.89 & 662.4 & 0.13 & \;\;\;-- & 0.03 & \;\;\;-- \\
	Gripper & \bf{1.00} & 79.5 & 0.00 & \;\;\;-- & 0.20 & 0.8 & 0.00 & \;\;\;-- & 0.00 & \;\;\;-- & 0.06 & \;\;\;-- & 0.00 & \;\;\;-- \\
	Miconic & \bf{1.00} & 434.8 & 0.00 & \;\;\;-- & 0.10 & 316.6 & 0.00 & \;\;\;-- & 0.90 & 2415.3 & 0.12 & \;\;\;-- & 0.13 & \;\;\;-- \\
	Ferry & \bf{1.00} & 16.1 & 0.00 & \;\;\;-- & 0.90 & 1.8 & 0.00 & \;\;\;-- & \bf{1.00} & 5705.2 & 0.11 & 11.8 & 0.00 & \;\;\;-- \\
	Spanner & \bf{1.00} & 17.7 & \bf{1.00} & 6.8 & 0.56 & 2.1 & \bf{1.00} & 31.2 & \bf{1.00} & 5.2 & 0.37 & 18.1 & 0.06 & \;\;\;-- \\\hline
	\end{tabular}
	\caption{\textbf{Policy learning results}. Eval columns report the fraction of test problems solved by the final learned policy. Time columns report the average wall-clock time (in seconds) required to learn a policy that solves $\ge 90\%$ of test problems, with a missing entry if such a policy was never found.
	All entries are means across 10 random seeds and 30 test problems per seed, with standard deviations shown in Table~\ref{tab:mainresultsstd}. }
	\label{tab:mainresults}
\end{table*}

\subsection{Generalized Policy Search Operators}

Here we present the search operators that we use for GPS with lifted decision list policies (Algorithm~\ref{alg:pg3}).
Recall each operator $\omega \in \Omega$ is a function from a policy $\pi$, a domain $\langle P, A \rangle$, and training problems $\Psi$, to a set of \emph{successor policies} $\Pi'$.

\textbf{Add Rule.} This operator adds a new rule to the given policy.
One successor policy is proposed for each possible new rule and each possible position in the decision list.
Each new rule $\rho$ corresponds to an action with preconditions from the domain, and no goal conditions: for $a \in A$, $\rho = \langle \textsc{Par}(a), \textsc{Pre}(a), \emptyset, a \rangle$.
The branching factor for this operator is therefore $|A|(|\pi|+1)$, where $|\pi|$ denotes the number of rules in the given decision list.

\textbf{Delete Rule.} This operator deletes a rule from the given policy. The branching factor for this operator is thus $|\pi|$.


\textbf{Add Condition.} For each rule $\rho$ in the given policy, this operator adds a literal to $\textsc{Pre}(\rho)$ or $\textsc{Goal}(\rho)$.
The literal may be positive or negative and the predicate may be any in $P$.
The terms can be any of the variables that are already in $\rho$.
The branching factor for this operator is thus $O(4|\pi||P|k^m)$, where $m$ is the maximum arity of predicates in $P$ and $k$ is the number of variables in $\rho$.

\textbf{Delete Condition.} For each rule $\rho$ in the given policy, this operator deletes a literal from $\textsc{Pre}(\rho)$ or $\textsc{Goal}(\rho)$.
Literals in the action preconditions are never deleted. The branching factor is therefore at most $\sum_{\rho \in \pi} |\textsc{Pre}(\rho)| + |\textsc{Goal}(\rho)|$.

\textbf{Induce Rule from Plans.} This final operator is the most involved; we describe it here at a high level and give details in Appendix \ref{sec:triangletablesopt}.
Unlike the others, this operator uses plans generated on the training problems to propose policy modifications.
In particular, the operator identifies a state-action pair that disagrees with the candidate policy, and amends the policy so that it agrees.
The mechanism for amending the policy is based on an extension of the \emph{triangle tables} approach of \cite{triangletables}.
This operator proposes a single change to each candidate, so the branching factor is 1.

Given these operators, we perform GPS starting with an empty lifted decision list using the operator order: Induce Rule from Plans; Add Condition; Delete Condition; Delete Rule; Add Rule.
In Appendix~\ref{sec:policysearchopt}, we describe two optimizations to improve GPS, irrespective of score function.


\section{Experiments and Results}
The following experiments evaluate the extent to which \Ours{} can learn policies that generalize to held-out test problems that feature many more objects than seen during training.


\textbf{Experimental Setup.}
Our main experiments consider the fraction of test problems solved by learned policies.
Note that here we are evaluating whether policies are capable of solving the problems on their own; we are not using the policies as planning guidance.
Each policy is executed on each test problem until either the goal is reached (success); the policy is not applicable in the current state (failure); or a maximum horizon is exceeded (failure) (Appendix~\ref{sec:domaindetails}).
All experimental results are over 10 random seeds, where training and test problem instances are randomly generated for each seed.
All GPS methods are run for 2500 node expansions.


\textbf{Domains.} 
We use the following domains:
\begin{tightlist}
\item \textbf{Forest}: See Figure \ref{fig:teaser} for description.
\item \textbf{Delivery}: An agent must deliver packages to multiple locations while avoiding trap locations. The agent can pick up any number of packages from a home base, move between locations, and deliver a package to a location.
\item \textbf{Miconic}: In this International Planning Competition (IPC) domain, an agent can move an elevator up or down and board or depart passengers. Each problem requires the agent to transport passengers to their desired floors.
\item \textbf{Gripper}: In this IPC domain, an agent with grippers can pick, move, and drop balls. Each problem requires the agent to transport balls to their target locations.
\item \textbf{Ferry}: In this IPC domain, an agent can sail, board, and debark cars. Each problem requires the agent to board cars, sail to their destination, and debark them.
\item \textbf{Spanner}: In this IPC domain, an agent can move, pick up spanners, and tighten nuts using spanners. Each problem requires the agent to move along the corridor and pick up spanners needed to tighten nuts at the shed at the end of the corridor. The agent cannot move backwards.
\end{tightlist}
See Appendix~\ref{sec:domaindetails} for problem counts, sizes, and more details.

\subsection{Approaches}
We evaluate the following methods and baselines:
\begin{tightlist}
\item \textbf{\Ours{} (Ours)}: Our main approach.
\item \textbf{Policy evaluation}: GPS with policy evaluation.
\item \textbf{Plan comparison}: GPS with plan comparison. Plans are collected on the training problems using $\text{A}^*$ with $\text{h}_{\text{Add}}$.
\item \textbf{Combo}: GPS with policy evaluation, but with ties broken using the plan comparison score function.
\item \textbf{Goal count}: GPS with a score function that runs the candidate policy on each training problem and counts the number of goal atoms not yet achieved at the end of execution. Based on prior work~\cite{segovia2021generalized}.
\item \textbf{Graph neural network behavior cloning (GNN BC)}: Inspired by recent works that use graph neural networks (GNNs) in PDDL domains, this method learns a GNN policy via behavior cloning using plans from the training problems.
The architecture and training method are taken from \cite{ploi}; see Appendix~\ref{sec:gnndetails}.
\item \textbf{Random}: a trivial baseline that selects among applicable actions uniformly at random on the test problems.
\end{tightlist}

\subsection{Results and Discussion}

See Table~\ref{tab:mainresults} for our main empirical results.
The results confirm that \Ours{} is able to efficiently guide policy search toward policies that generalize to larger held-out test problems.
Qualitatively, the policies learned by \Ours{} are compact and intuitive.
We highly encourage the reader to refer to Appendix~\ref{sec:learnedpolicies}, which shows a learned policy for each domain. 

Baseline performance is mixed, in most cases failing to match that of \Ours{}.
The strongest baseline is goal count, which confirms findings from previous work~\cite{segovia2021generalized}; however, even in domains with competitive final evaluation performance, the time required to learn a good policy can substantially exceed that for \Ours{}.
This is because goal count suffers from the same limitation as policy evaluation, but to a lesser degree: all policies will receive trivial scores until one is found that reaches at least one goal atom.
Plan comparison has consistently good performance only in Ferry; in this domain alone, the plans generated by the planner are consistent with the compact policy that is ultimately learned.
GNN BC is similarly constrained by the plans found by the planner.
Combo improves only marginally on policy evaluation and plan comparison individually.

\section{Conclusion}
In this work, we proposed \Ours{} as a new approach for generalized planning.
We demonstrated theoretically and empirically that \Ours{} outperforms alternative formulations of GPS, such as policy execution and plan comparison, and found that it is able to efficiently discover compact policies in PDDL domains.
There are many interesting directions for future work, including applying \Ours{} in domains with stochasticity and partial observability, and integrating insights from other approaches to generalized planning.

\section{Acknowledgements}
We gratefully acknowledge support from NSF grant 1723381; from AFOSR grant FA9550-17-1-0165; from ONR grant N00014-18-1-2847; and from MIT-IBM Watson Lab. Tom is supported by an NSF Graduate Research Fellowship.
Any opinions, findings, and conclusions or recommendations expressed in this material are those of the authors and do not necessarily reflect the views of our sponsors. 


\bibliographystyle{named}
\bibliography{ijcai22}


\appendix
\section{Proofs}
\label{sec:proofs}

\subsection{Proof of Theorem~\ref{thm:equal}}
\equalthm*
\begin{proof}
Let $\langle O, I, G \rangle$ be the problem in $\Psi$ that $\pi$ does not solve.
First we will show that the \Ours{} score (Algorithm~\ref{alg:pg3score}) is less than or equal to the GPS cost-to-go (Definition~\ref{defn:costtogo}).
Suppose that the \Ours{} score is $n$.
Then for the problem $\langle O, I, G \rangle$, there exists a plan $(\underline{a}_1, \dots, \underline{a}_n)$ with corresponding states $(S_1, \dots, S_{n+1})$, and indices $\mathcal{I} = \{i_1, \dots, i_n\}$ such that $\forall j \in \mathcal{I},$ $\pi(S_j) \neq \underline{a}_j$, and $\forall k \not\in \mathcal{I},$ $\pi(S_k) = \underline{a}_k$.
Consider the following tabular policy:
$$\{(S_j, G) \mapsto \underline{a}_j  : j \in \mathcal{I}\} \cup \{(S, G) \mapsto \pi(S) : S \not\in S_{\mathcal{I}}\}$$
where $S_{\mathcal{I}} = \{S_j : j \in \mathcal{I}\}$.
By construction, this policy is satisficing.
Furthermore, this policy differs from $\pi$ in $n$ table entries, so the cost-to-go is no more than $n$.

In the other direction, suppose towards a contradiction that the \Ours{} score is again $n$, but the GPS cost-to-go is $n' < n$.
Observe that the successor function used in policy-guided planning induces a modified search graph with respect to the search graph in the original problem.
Whereas all edge costs in the original search graph are unitary by Assumption~\ref{assumption:allplanscomparison}, the modified search graph contains new zero-cost edges corresponding to the policy's actions.
Since the GPS cost-to-go is $n'$, there exists a satisficing policy $\pi'$ that differs from $\pi$ in $n'$ states.
Let $(\underline{a}_1, \dots, \underline{a}_n)$ be the plan for policy $\pi'$ with corresponding states $(S_1, \dots, S_{n+1})$.
Since the policy is deterministic, these states do not have repeats (otherwise the policy could not be satisficing).
The cost of this plan in the modified search graph is therefore equal to $n'$.
However, by supposition, the minimum cost path in the modified search graph is $n$, a contradiction.
Therefore we can conclude that the \Ours{} score is equal to the cost-to-go.
\end{proof}

\subsection{Proof of Theorem~\ref{thm:costtogo}}
\admissiblethm*
\begin{proof}
Suppose to the contrary that PG3 is not GPS admissible.
Then there exists a policy $\pi$ and training problems $\Psi$ where the PG3 score $n$ exceeds the GPS cost-to-go $n'$.
Since PG3 accumulates per-plan scores using a \emph{max} (Algorithm~\ref{alg:pg3score}), there is a corresponding problem $\langle O, I, G \rangle$ in $\Psi$ that represents the \emph{argmax}.
In other words, the PG3 score for $\{\langle O, I, G \rangle \}$ is $n$. From Theorem~\ref{thm:equal}, we know that the GPS cost-to-go for $\pi$ on $\{\langle O, I, G \rangle \}$ is equal to $n$.
Furthermore, the GPS cost-to-go $n'$ for the full set $\Psi$ cannot be less than $n$, since each additional problem contributes non-negatively to the GPS cost-to-go in the tabular setting.
Thus $n \le n'$, a contradiction.
\end{proof}


\section{Induce Rule from Plans Operator}
\label{sec:triangletablesopt}

Here we describe in detail the final search operator, with a detailed walkthrough in the next subsection.
This search operator proposes a modification to a given policy based on example plans.
The plans are cached from the score function when possible; for example, \Ours{} uses the policy-guided plans, and plan comparison uses the planner-generated plans.
This operator suggests a modification to the candidate policy by identifying the last \emph{missed action} --- an action that does not match the candidate's output at the corresponding \emph{missed state} in one of the plans --- and computing a new rule for the policy that would ``correct'' the policy on that example.

New rules are computed via an extension of the \emph{triangle tables} approach of \cite{triangletables}, which was originally proposed for macro action synthesis.
Given a plan subsequence, triangle tables creates a macro by computing a \emph{preimage} for (an edited subsequence of) the subsequence.
The plan subsequence that we use as input to triangle tables is one that begins with the missed action and ends with the first subsequent goal literal that is permanently achieved, that is, added to the state and never removed in the remainder of the plan.
The preimage returned by triangle tables comprises the preconditions $\textsc{Pre}(\rho')$ of the new rule $\rho'$; the goal conditions $\textsc{Goal}(\rho')$ consist of the single goal literal at the tail of the subsequence; and the action $\textsc{Act}(\rho')$ is derived from the missed action.
The new rule $\rho'$ is inserted into $\pi$ before the first existing rule that was applicable in the missed state; if none were applicable, the new rule is appended to the end.

A limitation of the approach described thus far is that the preconditions added to the new rule can be overly specific.
We would prefer to introduce a more general rule, to promote generalization to new problems, so long as that rule still returns the missed action in the missed state.
With that intuition in mind, we implement a final extension of the triangle tables approach.
Instead of starting with the entire preimage as the precondition set for the rule, we start with a smaller set, computed as follows. All preconditions in the preimage which only involve objects in the failed action and the first permanently-achieved subsequent goal literal are included in proposed rule. If upon checking that the proposed rule does not produce the missed action, we iteratively step through each action in the plan subsequence used in the triangle table and add new preconditions to that rule, attempting to reproduce the missed action. The new preconditions that we add at each iteration are determined as follows: we add all preconditions involving objects from the previously proposed rule and the current action in the plan subsequence. If the missed action is never produced, the final proposed rule whose preconditions are equivalent to the entire preimage is returned.

In preliminary experiments, we found that this operator is very helpful in most domains, such as the Ferry example shown in the following section.
But in the Forest domain, the operator can lead to overfitting and syntactically large policies.
We therefore disable this operator for Forest; note that this change applies to all GPS methods, including \Ours{}, policy evaluation, plan comparison, combo, and goal count.

\subsection{Example}
\label{sec:induceexample}

\begin{figure}[h]
\centering
\includegraphics[scale=0.4]{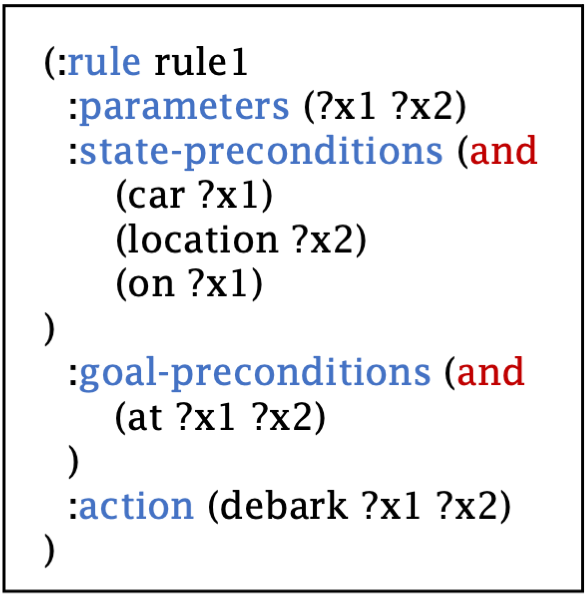}
\caption{Example input to Induce Rule from Plans operator; see Section \ref{sec:induceexample}}.
\label{fig:induce1}
\end{figure}

We now provide an example of the Induce Rule from Plans operator in the Ferry domain.
Consider the policy in Figure \ref{fig:induce1} as input to the operator, and consider the following plan:
\begin{enumerate}
    \item \texttt{sail(l8, l7)}
    \item \texttt{board(c0, l7)}
    \item \texttt{sail(l7, l0)}
    \item \texttt{debark(c0, l0)}
    \item \texttt{sail(l0, l2)} 
    \item \texttt{board(c4, l2)}
    \item \texttt{sail(l2, l8)}
    \item \texttt{debark(c4, l8)}
\end{enumerate}
for a problem where $G = \{\texttt{at(c4, l8), at(c0, l0)}\}$. The initial state is too large to print, but includes \texttt{atFerry(l8), at(c0, l7), at(c4, l2)} among others.

Referring to the policy, the last missed action in this plan is \texttt{sail(l2, l8)}.
The subsequent goal literal that is permanently achieved is \texttt{at(c4, l8)}, and it is achieved on the final time step, so the input to triangle tables is:
\begin{enumerate}
    \setcounter{enumi}{6}
    \item \texttt{sail(l2, l8)}
    \item \texttt{debark(c4, l8)}
\end{enumerate}
The preimage then returned by triangle tables is:

\begin{verbatim}{(car c4), (location l2), (location l8), 
(notEq l2, l8), (not (at l2, l8)), (atFerry l2), 
(on c4)}
\end{verbatim}

Replacing all objects with placeholders, we arrive at a new rule, which we then insert into the policy following the existing rule; see Figure \ref{fig:induce2}.
In this example, the additional generalization step described in the section above does not change the policy, so the policy in Figure \ref{fig:induce2} is the final output of the operator.

\begin{figure}[h]
\centering
\includegraphics[scale=0.4]{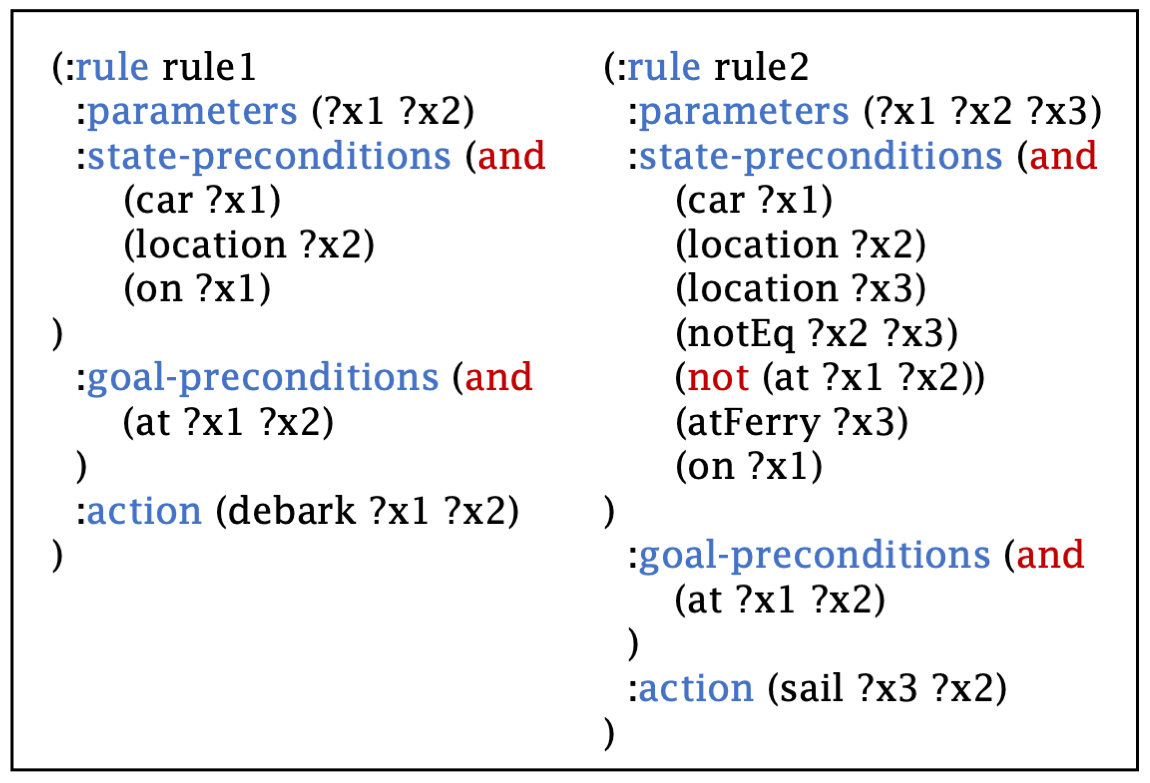}
\caption{Example output from Induce Rule from Plans operator; see Section \ref{sec:induceexample} and Figure \ref{fig:induce1}.}
\label{fig:induce2}
\end{figure}


\section{Additional GPS Optimizations}
\label{sec:policysearchopt}

There are two additional optimizations that we use in experiments to accelerate GPS.
First, in addition to the score functions described previously, we add a small \emph{size penalty} to the score: $$\textsc{Penalty}(\pi) = -w \sum_{\rho \in \pi}|\textsc{Pre}(\rho)| + |\textsc{Goal}(\rho)|$$
where $w$ is a weight selected to be small enough ($0.00001$ in experiments) so that the penalty only breaks ties in the original score function.
This regularization encourages GPS to discover policies with shorter description length, encoding an Occam's razor bias.

The second optimization is motivated by the observation that two policies may be syntactically distinct, but semantically equivalent with respect to all states and goals of interest.
This leads to two challenges: (1) it is difficult to compute semantic equivalence, especially with respect to the unknown set of states and goals of interest; and (2) we cannot necessarily prune policies according to semantic equivalence because their neighbors in policy search space may be semantically different.
Nonetheless, we would like to bias policy search so that it avoids redundancy.

To address the challenge of computing semantic equivalence, we resort to the following approximate approach.
At the beginning of GPS, we use a planner to collect a set of example plans on the training problems (as in the plan comparison scoring function).
The states and goals encountered in these example plans are used to approximate the ``states and goals of interest.''
During GPS, we execute each candidate policy on each of these states and goals to compute a \emph{semantic identifier}.
We maintain a map from semantic identifiers to counts, with the counts representing the number of times that a policy with that semantic identifier has been expanded.
When a new policy is generated, it is inserted into the GBFS queue (Algorithm~\ref{alg:pg3}) with priority (count, score), with the effect being that semantically unique policies are always explored first.
This optimization maintains the completeness of GPS while favoring policies that are approximately semantically unique.

\begin{table*}
	\centering
	\footnotesize
	\begin{tabular}{| l | p{0.7cm} | p{0.7cm} | p{0.7cm} | p{0.7cm} | p{0.7cm} | p{0.7cm} | p{0.7cm} | p{0.7cm} | p{0.7cm} | p{0.7cm} | p{0.7cm} | p{0.7cm} | p{0.7cm} | p{0.7cm} | }
	\hline
	\multicolumn{1}{|c|}{} &\multicolumn{2}{c|}{PG3 (Ours)} &
	\multicolumn{2}{c|}{Policy Eval} &
	\multicolumn{2}{c|}{Plan Comp} &
	\multicolumn{2}{c|}{Combo} &
	\multicolumn{2}{c|}{Goal Count} &
	\multicolumn{2}{c|}{GNN BC} &
	\multicolumn{2}{c|}{Random} \\
	\hline
	Domains &
	\;{\scriptsize Eval} & \:{\scriptsize Time} &
	\;{\scriptsize Eval} & \:{\scriptsize Time} &
	\;{\scriptsize Eval} & \:{\scriptsize Time} &
	\;{\scriptsize Eval} & \:{\scriptsize Time} &
	\;{\scriptsize Eval} & \:{\scriptsize Time} &
	\;{\scriptsize Eval} & \:{\scriptsize Time} &
	\;{\scriptsize Eval} & \:{\scriptsize Time} \\
	\hline
	Delivery & 0.00 & 0.8 & 0.00 & \;\;\;-- & 0.32 & \;\;\;-- & 0.42 & 1659.3 & 0.00 & 19.4 & 0.41 & 27.8 & 0.00 & \;\;\;-- \\
	Forest & 0.00 & 80.0 & 0.00 & 472.2 & 0.30 & \;\;\;-- & 0.00 & 689.2 & 0.33 & 517.6 & 0.07 & \;\;\;-- & 0.03 & \;\;\;-- \\
	Gripper & 0.00 & 20.8 & 0.00 & \;\;\;-- & 0.42 & 0.2 & 0.00 & \;\;\;-- & 0.00 & \;\;\;-- & 0.13 & \;\;\;-- & 0.00 & \;\;\;-- \\
	Miconic & 0.00 & 606.1 & 0.00 & \;\;\;-- & 0.30 & \;\;\;-- & 0.00 & \;\;\;-- & 0.31 & 1810.6 & 0.12 & \;\;\;-- & 0.05 & \;\;\;-- \\
	Ferry & 0.00 & 5.2 & 0.00 & \;\;\;-- & 0.32 & 0.3 & 0.00 & \;\;\;-- & 0.00 & 2340.0 & 0.18 & \;\;\;-- & 0.00 & \;\;\;-- \\
	Spanner & 0.00 & 15.9 & 0.00 & 3.0 & 0.53 & 0.4 & 0.00 & 17.6 & 0.00 & 2.0 & 0.42 & 12.5 & 0.04 & \;\;\;-- \\\hline
	\end{tabular}

	\caption{\textbf{Standard deviations for policy learning results.} See Table~\ref{tab:mainresults} for further description.}
	\label{tab:mainresultsstd}
\end{table*}

\section{Additional Experimental Details}
\label{sec:domaindetails}

All code is implemented in Python and run on a quad-core AMD64 processor with 16GB RAM, in Ubuntu 18.04.
We use PDDLGym~\cite{pddlgym} for simulation.

Below are additional details for each domain and train/test problem distributions, which are procedurally generated separately for each random seed. The specific problems used in this work are supplied with the code.

\begin{itemize}
    \item \textbf{Forest}
    
    \begin{itemize}
        \item \emph{Max horizon}: 100
        \item \emph{Train problems}: 10 problems, which have grids with heights and widths ranging from 8 to 10, for a total of \textbf{64--100 objects}.
        \item \emph{Test problems}: 30 problems, which have grids with heights and widths ranging from 10 to 12, for a total of \textbf{100--144 objects}.
    \end{itemize}
    
    \item \textbf{Delivery}
    \begin{itemize}
        \item \emph{Max horizon}: 1000
        \item \emph{Train problems}: 5 problems, which have 5--9 locations, 2--4 delivery requests, and 2--4 packages, for a total of \textbf{9--17 objects}.
        \item \emph{Test problems}: 30 problems, which have 30--39 locations, 20--29 requests, and 0--10 more packages than requests, for a total of \textbf{70--110 objects}.
    \end{itemize}
    \item \textbf{Miconic}
    \begin{itemize}
        \item \emph{Max horizon}: 1000
        \item \emph{Train problems}: 10 problems, which have 1--2 buildings, and for each building, 5--10 floors and 1--5 passengers, for a total of \textbf{6--30 objects}.
        \item \emph{Test problems}: 30 problems, which have 1--5 buildings, and for each building, 10--20 floors and 1--10 passengers, for a total of \textbf{11--150 objects}.
    \end{itemize}
    \item \textbf{Gripper}
    \begin{itemize}
        \item \emph{Max horizon}: 1000
        \item \emph{Train problems}: 10 problems, which have 5--10 balls and 15--20 rooms for a total of \textbf{20--30 objects}.
        \item \emph{Test problems}: 30 problems, which have 20--30 balls and 40--50 rooms for a total of \textbf{60--80 objects}.
    \end{itemize}
    \item \textbf{Ferry}
    \begin{itemize}
        \item \emph{Max horizon}: 1000
        \item \emph{Train problems}: 10 problems, which have 10--15 locations and 3--5 cars, for a total of \textbf{13--20 objects}.
        \item \emph{Test problems}: 30 problems, which have 20--30 locations and 10--20 cars, for a total \textbf{30--50 objects}.
    \end{itemize}
    \item \textbf{Spanner}
    \begin{itemize}
        \item \emph{Max horizon}: 1000
        \item \emph{Train problems}: 10 problems, which have 3--5 spanners, 3--5 nuts, and 3--5 locations, for a total of \textbf{9--15 objects}.
        \item \emph{Test problems}: 30 problems, which have 10--20 spanners, 10--20 nuts, and 10--20 locations, for a total of \textbf{30--60 objects}.
    \end{itemize}
    
\end{itemize}

\section{Additional Results}

\subsection{Standard Deviations for Main Results}

See Table~\ref{tab:mainresultsstd} for standard deviations on the main policy learning results presented in Table~\ref{tab:mainresults}.

\subsection{Qualitative Results: Learned Policies}
\label{sec:learnedpolicies}

See Figures 
\ref{fig:forestpolicy}, \ref{fig:deliverypolicy}, \ref{fig:miconicpolicy},
\ref{fig:gripperpolicy},
\ref{fig:ferrypolicy},
\ref{fig:spannerpolicy}
for policies learned by \Ours{}.

\begin{figure*}[t]
\centering
\includegraphics[scale=0.5]{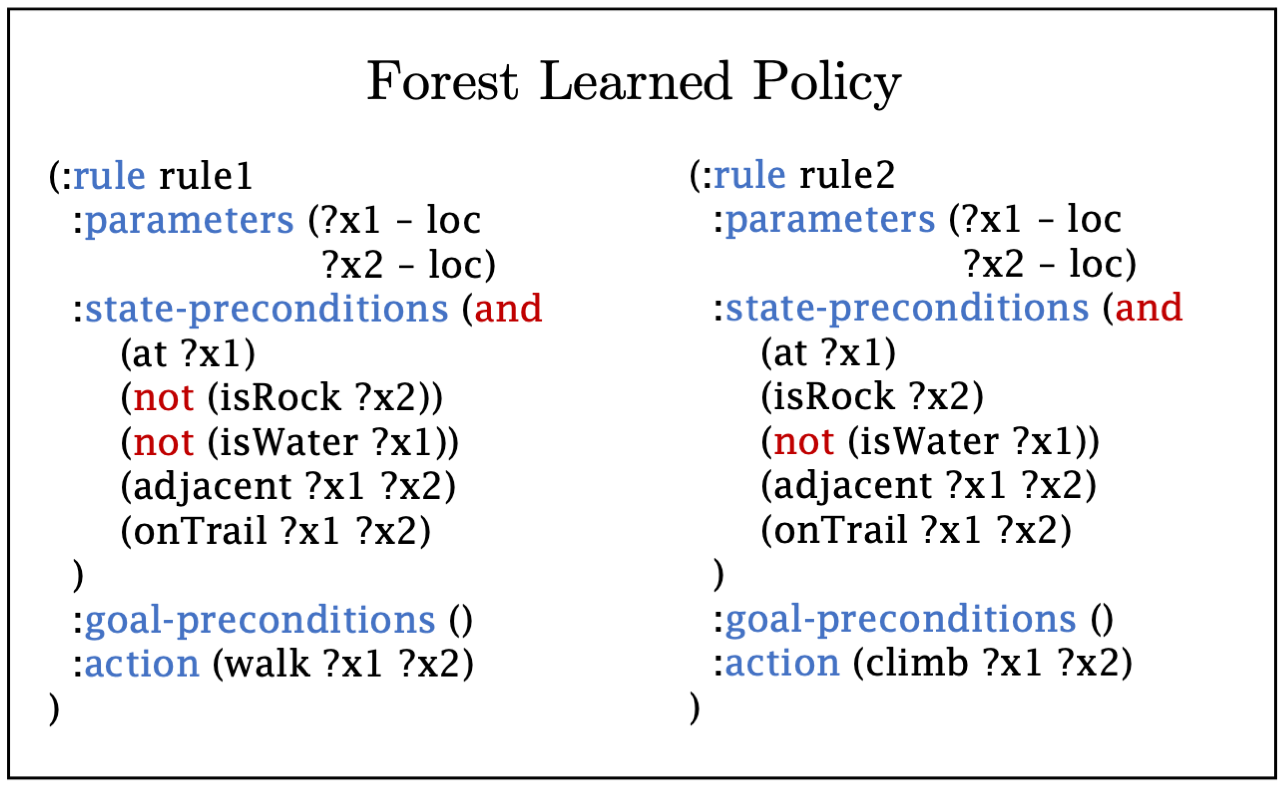}
\caption{\textbf{Learned policy for the Forest domain.} This policy has the agent follow the marked trail until a rock is encountered, at which point it climbs over the rock. For this policy and the others, it is important to remember that the rules are \emph{ordered} (left to right); each rule will only be used when the previous rules are not applicable.}
\label{fig:forestpolicy}
\end{figure*}

\begin{figure*}
\centering
\includegraphics[scale=0.5]{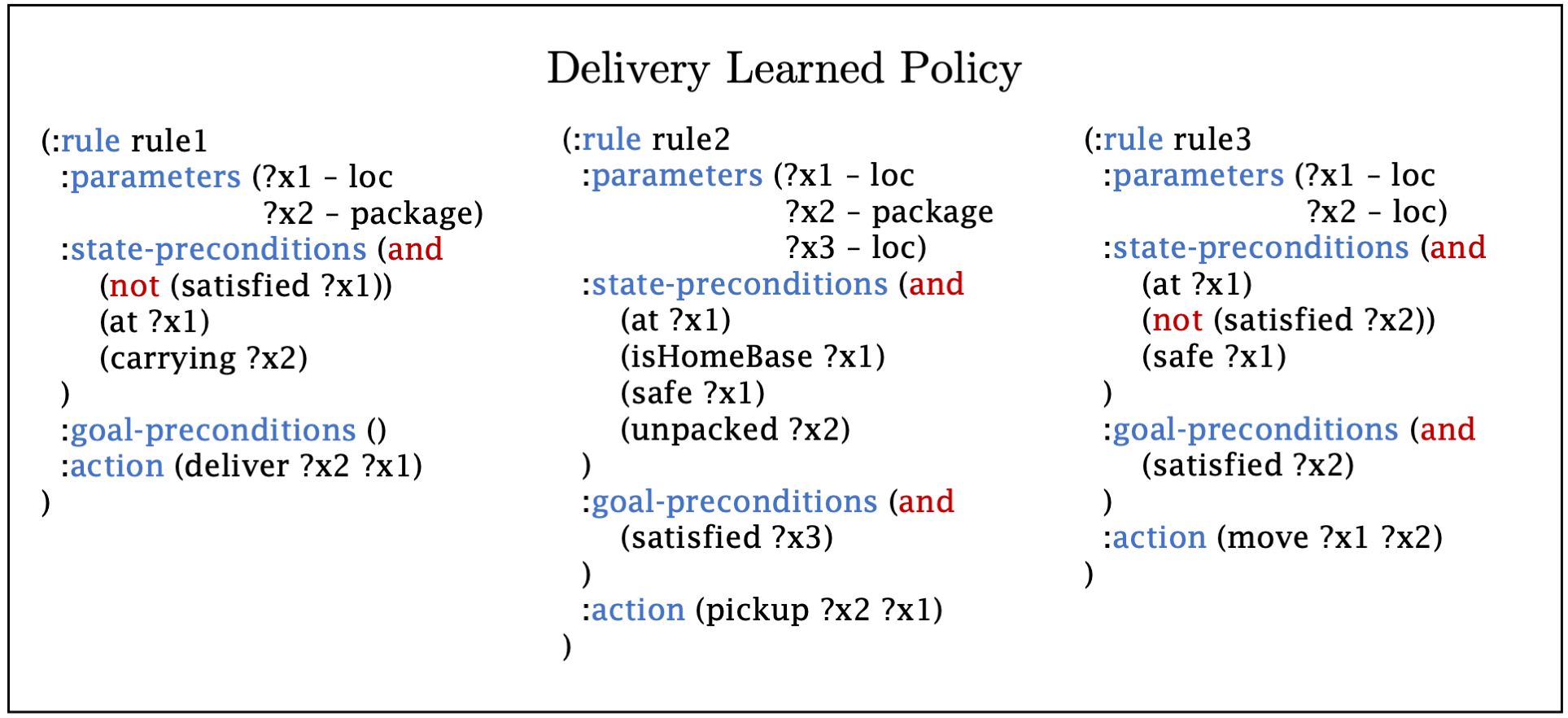}
\caption{\textbf{Learned policy for the Delivery domain.} This policy has the agent deliver a package whenever it is at a location that wants a package but does not yet have one; it also picks up packages whenever there are packages available at the home base; and it moves to a location that wants a package, but has not yet received one.}
\label{fig:deliverypolicy}
\end{figure*}

\begin{figure*}
\centering
\includegraphics[width=\textwidth]{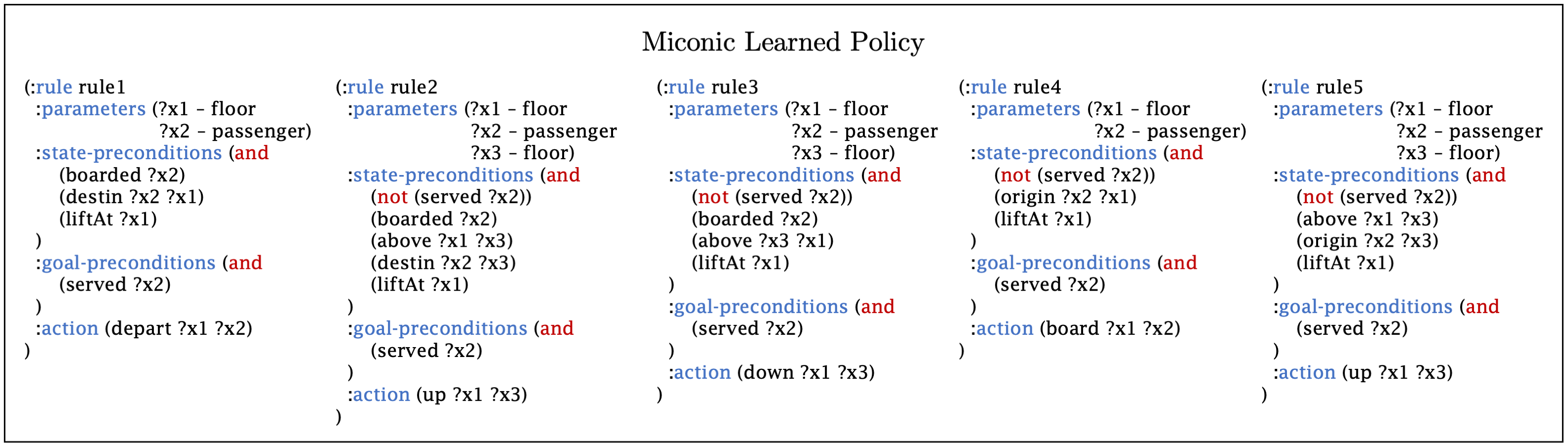}
\caption{\textbf{Learned policy for the Miconic domain.} This policy has the agent drop off passengers whenever they are boarded and at their destination floor; move up or down, depending on whether the destination floor is above or below, whenever a passenger is boarded but not yet at their destination floor; board a passenger whenever there is a passenger at the elevator's current floor that is not yet at their destination; and to start, moves up to a floor where there is a passenger waiting (the elevator always starts on the ground floor).}
\label{fig:miconicpolicy}
\end{figure*}

\begin{figure*}
\centering
\includegraphics[width=\textwidth]{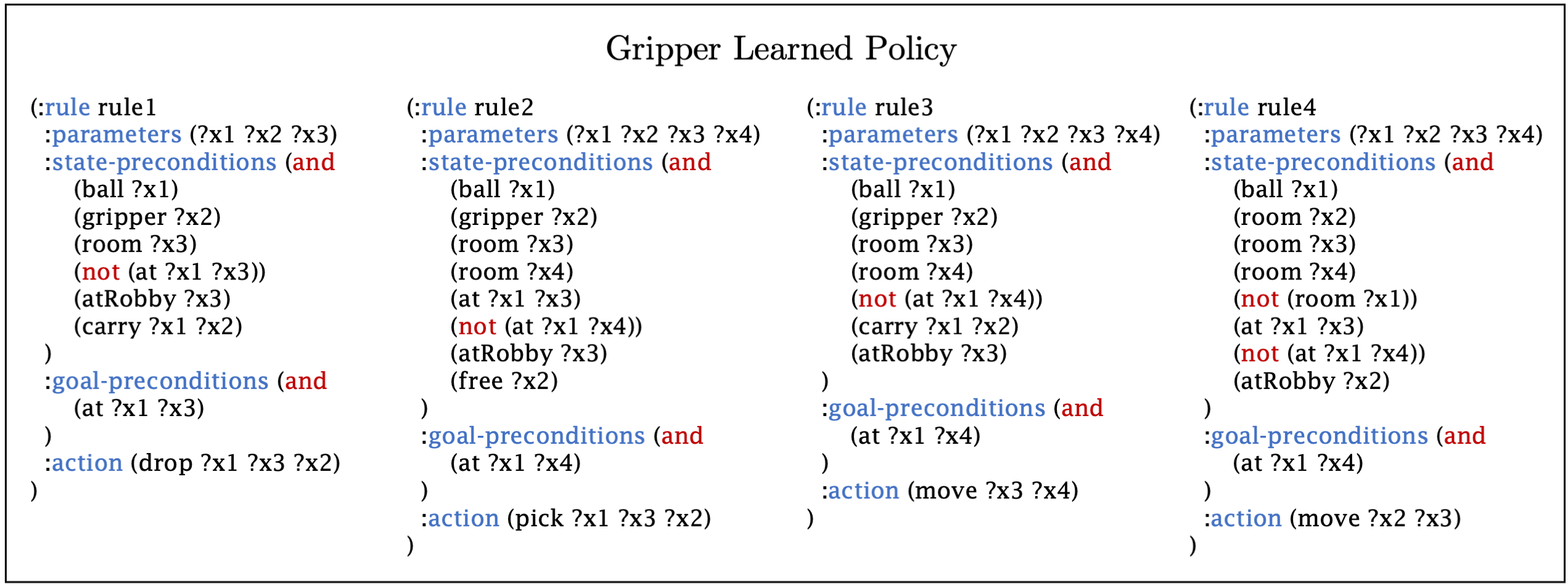}
\caption{\textbf{Learned policy for the Gripper domain.} Note that objects and parameters in this domain are untyped. This policy has the agent drop a ball that it's carrying when in the target room for that ball; pick a ball when the gripper is free and the ball is not yet in its target room; move to a target room when carrying a ball whose target room is not the robot's current room; and initially, moves to a room with some ball that is not yet at its target room.}
\label{fig:gripperpolicy}
\end{figure*}

\begin{figure*}
\centering
\includegraphics[width=\textwidth]{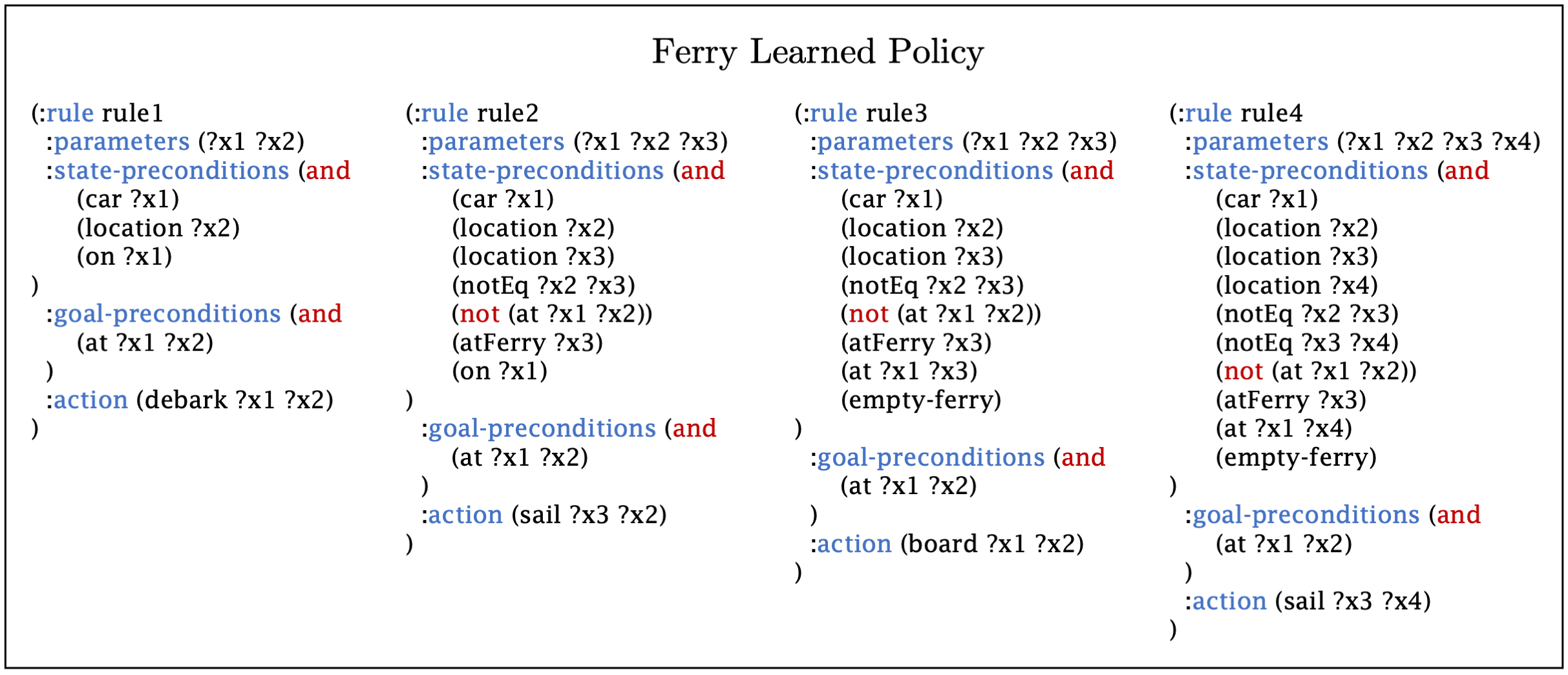}
\caption{\textbf{Learned policy for the Ferry domain.} Note that objects and parameters in this domain are untyped. This policy has the agent drop off (debark) cars whenever they are on the ferry but at their destinations; sail to a destination location when there is a car loaded; board a car when the ferry is empty and there is a car at the ferry's current location that has not yet been served; and sail to a location with a car that has not yet been served otherwise.}
\label{fig:ferrypolicy}
\end{figure*}

\begin{figure*}
\centering
\includegraphics[scale=0.5]{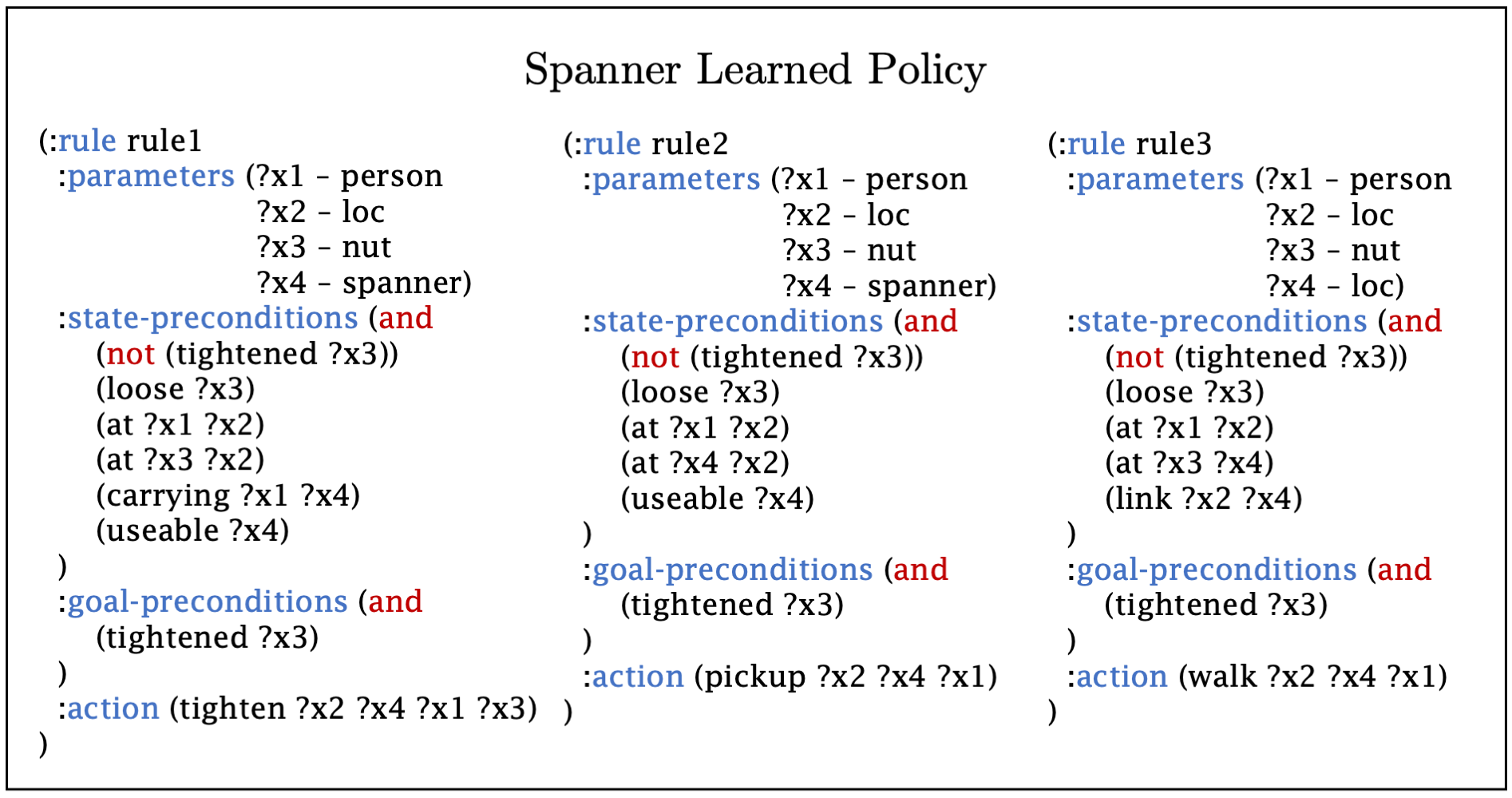}
\caption{\textbf{Learned policy for the Spanner domain.} This policy has the agent tighten a nut whenever there is a loose nut at the agent's location and the agent is carrying a spanner (wrench); pick up a spanner whenever it is available at the current location and not yet carried; and walk to a location with a loose nut otherwise.}
\label{fig:spannerpolicy}
\end{figure*}

\section{Graph Neural Network Baseline Details}
\label{sec:gnndetails}
The GNN behavioral cloning baseline is implemented in PyTorch, version 1.9.0. Our architecture and learning hyperparameters are identical to those used in the reactive policy baseline of~\cite{ploi}. All GNNs node and edge modules are fully connected neural networks with one hidden layer of dimension 16, ReLU activations, and layer normalization. Message passing is performed for $K=3$ iterations.
Training uses the Adam optimizer with learning rate $0.001$ for 100000 epochs with a batch size of 16. Performance is evaluated by directly executing the learned policy on the test problems for up to 1000 steps. The reported validation accuracy is the percent of test problem executions that reach the goal after the last training epoch.

\end{document}